%% file: fair_rl.tex
\documentclass[11pt]{article}
\usepackage[margin=1in]{geometry}
\usepackage[margin=0.2in]{caption}
\usepackage[font={scriptsize}]{caption}
\usepackage{xspace}
\usepackage{amsmath}
\usepackage{amsfonts}
\usepackage{graphicx}
\usepackage{mathtools}
\usepackage[numbers, sort]{natbib}
\usepackage[algo2e]{algorithm2e}
\usepackage{algorithm}
\usepackage{algpseudocode}
\usepackage{amsthm}
\usepackage[dvipsnames]{xcolor}
\usepackage{tikz}
\usepackage{color}
\usetikzlibrary{arrows,decorations.markings}
\usepackage{url}
\usepackage[hidelinks]{hyperref}
\hypersetup{breaklinks=true}
\def\bug{0}

\usepackage{etoolbox}
\newtoggle{condense}
\toggletrue{condense}

\title{\textbf{Fairness in Reinforcement Learning}}
\author{Shahin Jabbari, Matthew Joseph, Michael Kearns, Jamie Morgenstern, Aaron Roth\\
Department of Computer and Information Science, University of Pennsylvania\\
\{jabbari, majos, mkearns, jamiemor, aaroth\}@cis.upenn.edu
}
\date{}

\newtheorem{theorem}{Theorem}
\newtheorem*{theorem*}{Theorem}
\newtheorem{lemma}[theorem]{Lemma}

\newtheorem{corollary}[theorem]{Corollary}
\newtheorem{definition}{Definition}
\newtheorem{obs}{Observation}

\newtheorem{assumpt}{Assumption}

\DeclareMathOperator*{\argmax}{\arg\!\max}
\renewcommand{\Pr}{\mathbb{P}}
\newcommand{\ecube}{\textbf{E}$^3 $\xspace}
\newcommand{\ecubep}{\textbf{Fair-E}$^3$\xspace}
\newcommand{\theu}{an}
\newcommand{\rmax}{1}

\newcommand{\poly}{\textbf{poly}}
\newcommand{\E}{\mathbb{E}}
\newcommand{\alg}{\ensuremath{\mathcal{L}}\xspace}

\newcommand{\hist}{\ensuremath{h}}
\newcommand{\known}{\Gamma}
\newcommand{\A}{\mathcal{A}}
\newcommand{\He}{{H_\epsilon^\gamma}}
\newcommand{\Te}{T^*_\epsilon}

\newcommand{\T}{\mathcal{T}}
\newcommand{\afair}{\text{approximate-action}\xspace}
\newcommand{\pfair}{\text{approximate-choice}\xspace}

\newcommand{\mknown}{\ensuremath{M_\known}\xspace}
\newcommand{\munknown}{\ensuremath{M_{[n] \setminus \known}}\xspace}
\newcommand{\hmknown}{\ensuremath{\hat{M}_\known}\xspace}
\newcommand{\hmunknown}{\ensuremath{\hat{M}_{[n] \setminus \known}}\xspace}

\newcommand{\fgamma}{\frac{1}{1-\sqrt[3]{\gamma}}}

\newcommand{\jm}[1]{{\color{red}[Jamie: {#1}]}}

\newcommand{\citen}[1]{~[\citenum{#1}]}

\begin{document}

\maketitle

\begin{abstract} 
\input{abstract}
\end{abstract} 

\input{intro}
\input{contributions}
\input{related-work}
\input{prelim}
\input{fairness-notions}
\input{lower-bound}
\input{upper-bound}
\input{discussion}

\Urlmuskip=0mu plus 1mu\relax
\bibliographystyle{plainnat}
\bibliography{bib}

\appendix
\section{Omitted Proofs}
\input{missing-proofs-prelim}
\input{missing-proofs-lower-bound-new}
\input{missing-proofs-fair-ecube}
\input{fair-observations}
\input{fair-ecube-details}

\end{document}

%% file: abstract.tex
We initiate the study of {\em fairness\/} in reinforcement learning,
where the actions of a learning algorithm may affect its environment
and future rewards. 
\iffalse Working in the \iffalse well-studied\fi model of
reinforcement learning, \fi
Our fairness constraint requires that
an algorithm never prefers one action over another if the long-term
(discounted) reward of choosing the latter action is higher.
Our first result is negative: despite the fact that
fairness is consistent with the optimal policy, any
learning algorithm satisfying fairness must take time
exponential in the number of states
to achieve non-trivial approximation to the optimal policy.  
We then provide a provably fair polynomial time algorithm 
under an approximate notion of fairness, thus
establishing an exponential gap between exact and approximate
fairness.

%  We introduce the study of fairness in the setting of reinforcement
%  learning. We adapt our notion of fairness from~\citet{JosephKMR16},
%  which requires that over the course of its learning process a
%  learning algorithm never favors a worse action over a better one,
%  with high probability. Working in the setting of Markov Decision
%  Processes, we begin by proving an exponential separation in the time
%  required for fair and unfair algorithms to achieve near-optimal
%  performance. This separation motivates a relaxation to a new notion
%  of \emph{approximate fairness}, for which we prove more favorable
%  lower and upper bounds.

%% file: intro.tex
\newcommand{\R}{\mathbb{R}\xspace}
\renewcommand{\S}{\mathcal{S}\xspace}

\section{Introduction}
\label{sec:intro}

The growing use of machine learning for automated decision-making has
raised concerns about the potential for unfairness in learning
algorithms and models.  In settings as diverse as
policing~\citen{policing}, hiring~\citen{hiring},
lending~\citen{lending}, and criminal sentencing~\citen{sentencing},
mounting empirical evidence suggests these concerns are not merely
hypothetical~\citen{propublica,Sweeney13}.

We initiate the study of fairness in reinforcement learning, where an
algorithm's choices may influence the state of the world and future
rewards.  In contrast, previous work on fair machine learning has
focused on myopic settings where such influence is absent, e.g. in
i.i.d. or no-regret models\citen{DworkHPRZ12,FFMSV15,HardtPS16,
  JosephKMR16}.  The resulting fairness definitions therefore do not
generalize well to a reinforcement learning setting, as they do not
reason about the effects of short-term actions on long-term
rewards. This is relevant for the settings where historical
context can have a distinct influence on the future.  For
concreteness, we consider the specific example of hiring (though other
settings such as college admission or lending decisions can be
embedded into this framework).  Consider a firm aiming to hire
employees for a number of positions. The firm might consider a
variety of hiring practices, ranging from targeting and hiring
applicants from well-understood parts of the applicant pool (which
might be a reasonable policy for short-term productivity of its
workforce), to exploring a broader class of applicants whose
backgrounds might differ from the current set of employees at the
company (which might incur short-term productivity and learning costs
but eventually lead to a richer and stronger overall applicant pool).

We focus on the standard model of reinforcement learning, in which an
algorithm seeks to maximize its discounted sum of rewards in a
Markovian decision process (MDP). Throughout, the reader should
interpret the \emph{actions} available to a learning algorithm as
corresponding to choices or policies affecting individuals (e.g. which
applicants to target and hire). The \emph{reward} for each action
should be viewed as the short-term payoff of making the corresponding
decision (e.g. the short-term influence on the firm's productivity after
hiring any particular candidate). The actions taken by the algorithm
affect the underlying state of the system (e.g. the company's demographics
as well as the available applicant pool) and therefore in turn will
affect the actions and rewards available to the algorithm in the
future.

Informally, our definition of fairness requires that (with high
probability) in state $s$, an algorithm never chooses an available
action $a$ with probability higher than another action $a'$ unless
$Q^*(s,a) > Q^*(s,a')$, i.e. the long-term reward of $a$ is greater
than that of $a'$.  This definition, adapted from~\citet{JosephKMR16},
is \emph{weakly meritocratic}: facing some set of actions, an
algorithm must pick a distribution over actions with (weakly) heavier
weight on the better actions (in terms of their discounted long-term
reward). Correspondingly, a hiring process satisfying our fairness
definition cannot probabilistically target one population over another
if hiring from either population will have similar long-term benefit
to the firm's productivity.

Unfortunately, our first result shows an exponential separation in
expected performance between the best unfair algorithm and any
algorithm satisfying fairness. This motivates our study of a natural
relaxation of (exact) fairness, for which we provide a polynomial time
learning algorithm, thus establishing an exponential separation
between exact and approximately fair learning in MDPs.

%% file: contributions.tex
%\subsection{Our Contributions}
%\label{subsec:contributions}

\noindent {\bf Our Results}
Throughout, we use \emph{(exact)} fairness to refer to the
adaptation of~\citet{JosephKMR16}'s definition defining an action's
quality as its potential long-term discounted reward. We also
consider two natural relaxations.  The first, \emph{\pfair fairness},
requires that an algorithm never chooses a worse action with
\emph{probability} substantially higher than better actions.  The
second, \emph{\afair fairness}, requires that an algorithm never
favors an action of substantially lower \emph{quality} than that of a
better action.

The contributions of this paper can be divided into two
parts. 
First, in Section~\ref{sec:lower-bound}, we give a lower bound on the
time required for a learning algorithm to achieve near-optimality subject
to (exact) fairness or \pfair fairness.
\begin{theorem*}[Informal statement of
  Theorems~\ref{thm:lower},~\ref{thm:lower2}, and~\ref{thm:lower1}]
  For constant $\epsilon$, to achieve $\epsilon$-optimality, \emph{(i)} any fair or \pfair fair algorithm
  takes a number of rounds exponential in the number of MDP
  states and \emph{(ii)} any \afair fair algorithm
  takes a number of rounds exponential  in $1/(1-\gamma)$, for discount factor $\gamma$.
\end{theorem*}

Second, we present an \afair fair algorithm (\ecubep) in
Section~\ref{sec:upper-bound} and prove a polynomial upper bound on
the time it requires to achieve near-optimality.
\begin{theorem*}[Informal statement of Theorem~\ref{thm:fair-rl}]
  For constant $\epsilon$ and any MDP satisfying standard assumptions, \ecubep is an \afair
  fair algorithm achieving $\epsilon$-optimality in a number
  of rounds that is \ifnum\bug=1 \else (necessarily) \fi exponential
  in $1/(1-\gamma)$ and polynomial in other parameters.
\end{theorem*}

\ifnum\bug=1 Note that Theorem~\ref{thm:lower1} shows that it is
impossible for any \afair fair algorithm to have a running time that
has a polynomial dependence on $1/(1-\gamma)$.  \else The exponential
dependence of \ecubep on $1/(1-\gamma)$ is tight: it matches our lower
bound on the time complexity of any \afair fair algorithm.  \fi 
Furthermore, our results establish rigorous trade-offs
between fairness and performance facing reinforcement learning
algorithms.

%% file: related-work.tex
\subsection{Related Work}
\label{sec:related}
The most relevant parts of the large body of literature on reinforcement
learning focus on constructing learning algorithms with provable
performance guarantees. \ecube~\citen{KearnsS02} was the first learning
algorithm with a polynomial learning rate, and subsequent work improved
this rate (see~\citet{SzitaS10} and references within). The study of
\emph{robust} MDPs~\citen{MD05, MMX12, LXM13} examines MDPs with high parameter
uncertainty but generally uses ``optimistic" learning strategies that ignore
(and often conflict with) fairness and so do not directly apply to this work.

Our work also belongs to a growing literature studying the problem of
fairness in machine learning. Early work in data
mining~\citen{HajianD13, KamiranKZ12, Zem13, KamishimaAAS12,
  pedreshi08, LRT11} considered the question from a primarily
empirical standpoint, often using \emph{statistical parity} as a
fairness goal. ~\citet{DworkHPRZ12} explicated several drawbacks of 
statistical parity and instead proposed one of the first
broad definitions of algorithmic fairness, formalizing the idea that
``similar individuals should be treated similarly".
Recent papers have proven several
impossibility results for satisfying different fairness requirements
simultaneously~\citen{KleinbergMR16,FSV16}. 
More recently,~\citet{HardtPS16} proposed new notions of
fairness and showed how to achieve
these notions via post-processing of a black-box
classifier. 
\citet{WoodworthGOS17}~and~\citet{ZVGG17} further studied these notion theoretically and empirically.

\subsection{Strengths and Limitations of Our Models}

In recognition of the duration and consequence of choices made by a learning
algorithm during its learning process -- e.g. job applicants not hired -- our work departs from previous work
and aims to guarantee the fairness of \emph{the learning process itself}.
To this end, we adapt the fairness definition of~\citet{JosephKMR16},
who studied fairness in the bandit framework and defined fairness with respect
to one-step rewards. To capture the desired interaction and evolution of the
reinforcement learning setting, we modify this myopic definition and define
fairness with respect to long-term rewards: a fair learning algorithm
may only choose action $a$ over action $a'$ if $a$ has true long-term reward
at least as high as $a'$.
Our contributions thus depart from previous work in reinforcement
learning by incorporating a fairness requirement (ruling out existing
algorithms which commonly make heavy use of
``optimistic" strategies that violates fairness) and depart
from previous work in fair learning by requiring ``online" fairness in
a previously unconsidered reinforcement learning context.

First note that our definition is \emph{weakly meritocratic}: an algorithm
satisfying our fairness definition can \emph{never} probabilistically favor a worse
option but is not \emph{required} to favor a better option. This confers both
strengths and limitations. Our fairness notion still permits a type of ``conditional 
discrimination" in which a fair algorithm favors group A over group B by 
selecting choices from A when they are superior and
randomizing between A and B when choices from B are superior. In this sense, our
fairness requirement is relatively minimal, encoding a necessary variant of 
fairness rather than a sufficient one. This makes our
lower bounds and impossibility results (Section~\ref{sec:lower-bound})
relatively stronger and upper bounds (Section~\ref{sec:upper-bound}) relatively
weaker.

Next, our fairness requirement holds (with high probability) 
across \emph{all} decisions that a fair algorithm makes. We view this strong 
constraint as worthy of serious consideration, since ``forgiving" unfairness during the learning
may badly mistreat the training population, especially if the learning
process is lengthy or even continual. Additionally, it is unclear how to relax
this requirement, even for a small fraction of the algorithm's decisions,
without enabling discrimination against a correspondingly small population.

Instead, aiming to preserve the ``minimal" spirit of our definition,
we consider a relaxation that only prevents an algorithm from
favoring a \emph{significantly} worse option over a better option
(Section~\ref{sec:fairness}). Hence,
approximate-action fairness should be viewed as a weaker constraint:
rather than safeguarding against every violation of ``fairness'', it
instead restricts how egregious these violations can be. We discuss further
relaxations of our definition in Section~\ref{sec:future}.

%% file: prelim.tex
\section{Preliminaries}
\label{sec:prelim}

In this paper we study reinforcement learning in Markov Decision
Processes (MDPs). An MDP is a tuple
$M = (\mathcal{S}_M, \mathcal{A}_M, P_M, R_M, T, \gamma)$ where
$\mathcal{S}_M$ is a set of $n$ \emph{states}, $\A_M$ is a set of $k$
\emph{actions}, $T$ is a \emph{horizon} of a (possibly infinite)
number of rounds of activity in $M$, and $\gamma$ is a \emph{discount
  factor}.
$P_M: \mathcal{S}_M \times \mathcal{A}_M \to \mathcal{S}_M$ and
$R_M:\mathcal{S}_M \to [0,1]$ denote the \emph{transition probability
  distribution} and \emph{reward distribution}, respectively. We use
$\bar{R}_M$ to denote the mean of $R_M$.\footnote{Note that
  $\bar{R}_M\leq 1$ and $\text{Var}(R_M) \leq 1$ for all states. The
  bounded reward assumption can be relaxed
  (see~e.g.\citen{KearnsS02}).  Also assuming rewards in $[0, 1]$ can
  be made w.l.o.g. up to scaling. } A policy $\pi$ is a mapping from a
history $h$ (the sequence of triples (state, action, reward) observed
so far) to a distribution over actions. The discounted state and
state-action value functions are denoted by $V^\pi$ and $Q^\pi$, and
$V^\pi(s,T)$ represents expected discounted reward of following $\pi$
from $s$ for $T$ steps.  The highest values functions are achieved by
the \emph{optimal} policy $\pi^*$ and are denoted by $V^*$ and
$Q^*$~\citen{SuttonB98}.  We use $\mu^\pi$ to denote the stationary
distribution of $\pi$.  Throughout we make the following assumption.

\begin{assumpt}[Unichain Assumption]
\label{assumpt:unichain}
The stationary distribution of any policy in $M$ is independent of its
start state.
\end{assumpt}

We denote the \emph{$\epsilon$-mixing time} of $\pi$ by
$T^\pi_\epsilon$.  Lemma~\ref{lem:mixing2} relates the
$\epsilon$-mixing time of any policy $\pi$ to the number of rounds
until the $V^\pi_M$ values of the visited states by $\pi$ are close to
the expected $V^\pi_M$ values (under the stationary distribution
$\mu^\pi$). We defer all the omitted proofs to the Appendix.
\begin{lemma}
  \label{lem:mixing2} Fix $\epsilon > 0$.
  For any state $s$, following $\pi$ for $T\geq T^\pi_{\epsilon}$
  steps from $s$ satisfies
\[
\E_{s\sim\mu^\pi}\left[V^\pi_M(s)\right] - \E\left[\tfrac{1}{T}\sum_{t=1}^T V^\pi_M(s_t)\right] \leq \tfrac{\epsilon}{1-\gamma},
\]
where $s_t$ is the state visited at time $t$ when following $\pi$ from
$s$ and the expectation in the second term is over %(the randomization
%in
 the transition function
% of) $M$
 and the randomization of $\pi$.\footnote{Lemma~\ref{lem:mixing2} can
   be stated for a weaker notion of mixing time called the
   \emph{$\epsilon$-reward mixing time} which is always linearly
   bounded by the $\epsilon$-mixing time but can be much smaller in
   certain cases (see~\citet{KearnsS02} for a discussion).}
\end{lemma}

The \emph{horizon time} $\He \coloneqq \log\left(\epsilon(1-\gamma)\right)/\log(\gamma)$ of an MDP captures the number of steps
an approximately optimal policy must optimize over. 
The expected discounted reward of any policy after
$\He$
steps approaches the expected asymptotic discounted reward (\citet{KearnsS02}, Lemma 2).
A learning algorithm $\alg$ is a non-stationary policy that at each round
takes the entire history and outputs a distribution over actions.
We now define a performance measure for
learning algorithms.

\begin{definition}[$\epsilon$-optimality]
  \label{def:performance} Let $\epsilon > 0$ and $\delta\in(0, 1/2)$.
  \alg~achieves \emph{$\epsilon$-optimality} in $\T$ steps if for any $T\geq \T$ 
\begin{equation}
\label{eq:fair}
\E_{s\sim \mu^*}\left[V^*_M(s)\right]  - \E\left[\frac{1}{T}\sum_{t=1}^{T} V^*_M(s_t)\right]
\leq \frac{2\epsilon}{1-\gamma},
\end{equation}
with probability at least $1-\delta$, for $s_t$ the state
$\mathcal{L}$ reaches at time $t$, where the expectation is taken over
the transitions and the randomization of $\mathcal{L}$, for any MDP
$M$.
\end{definition}

We thus ask that a learning algorithm, after sufficiently many steps,
visits states whose values are arbitrarily close to the values of the
states visited by the optimal policy. Note that this is stronger than
the ``hand-raising'' notion in~\citet{KearnsS02},\footnote{We suspect
  unfair \ecube also satisfies this stronger notion.}  which only asked
that the learning algorithm stop in a state from which discounted
return is near-optimal, permitting termination in a state from which
the optimal discounted return is poor. In
Definition~\ref{def:performance}, if there are states with poor
optimal discounted reward that the optimal policy eventually leaves
for better states, so must our algorithms.  We also note the following
connection between the average $V^\pi_M$ values of states visited
under the stationary distribution of $\pi$ (and in particular
\theu~optimal policy) and the average undiscounted rewards achieved
under the stationary distribution of that policy.

\begin{lemma}[\citet{satinder-email}]
\label{lem:satinder}
Let $\bar{\textbf{R}}_M$ be the vector of mean rewards in states of
$M$ and $\textbf{V}^\pi_M$ the vector of discounted rewards in states
under $\pi$. Then
$\mu^\pi \cdot \bar{\textbf{R}}_M = (1-\gamma)\mu^\pi\cdot \textbf{V}^\pi_M.$
\end{lemma}

We design an algorithm which quickly achieves $\epsilon$-optimality and we bound the number of steps
$\T$ before this happens by a polynomial in the parameters of $M$. 

%% file: fairness-notions.tex
\newcommand{\qs}[2]{Q^*_M(#1,#2, 1)\xspace}

\subsection{Notions of Fairness}

\label{sec:fairness}
We now turn to formal notions of fairness.  Translated to our
setting, ~\citet{JosephKMR16} define action $a$'s quality as
the expected immediate reward for choosing $a$ from state $s$ and then
require that an algorithm not probabilistically favor $a$ over $a'$ 
if $a$ has lower expected immediate reward. 

However, this naive translation does not adequately capture the structural
differences between bandit and MDP settings since present rewards may
depend on past choices in MDPs. In particular, defining fairness
in terms of immediate rewards would prohibit any policy sacrificing
short-term rewards in favor of long-term rewards. This is undesirable,
since it is the long-term rewards that matter
in reinforcement learning, and optimizing for long-term rewards often
necessitates short-term sacrifices. 
Moreover, the
long-term impact of a decision should be considered when arguing about
its relative fairness. We will therefore define fairness using the
state-action value function $Q^*_M$.
\begin{definition}[Fairness]
\label{def:fair-ex}
$\alg$ is \emph{fair} if for all input $\delta> 0$, all $M$, all
rounds $t$, all states $s$ and all actions $a, a'$
\begin{equation*}
  Q^*_M(s,a) \geq Q^*_M(s,a') \Rightarrow \alg(s, a, \hist_{t-1}) \geq \alg(s, a', \hist_{t-1})
\end{equation*}
with probability at least $1-\delta$ over histories
$h_{t-1}$. \footnote{{\scriptsize $\alg(s, a, \hist)$ denotes the probability
  $\alg$ chooses $a$ from $s$ given history $h$.}}
\end{definition}

Fairness requires that an algorithm \emph{never} probabilistically favors an action with
lower long-term reward over an action with higher long-term reward. In hiring, this
means that an algorithm cannot target one applicant population over another unless
the targeted population has a higher quality.

\iffalse
If different ad copy results in applicant pools with different
demographic properties but identical skillsets and identical future
applicant qualities, any algorithm satisfying this (exact) fairness
constraint cannot probabilistically prefer one version of the ad copy
(and corresponding demographic group) over the other.
\fi

In Section~\ref{sec:lower-bound}, we show that fairness can be
extremely restrictive. Intuitively,
$\alg$ must play uniformly at random until it has high
confidence about the $Q^*_M$ values, in some cases taking exponential
time to achieve near-optimality. This motivates relaxing
Definition~\ref{def:fair-ex}. We first relax the \emph{probabilistic} requirement and
require only that an algorithm not \emph{substantially} favor a worse action over a better
one.

\begin{definition}[Approximate-choice~Fairness]
\label{def:fair-approx2}
$\mathcal{L}$ is \emph{$\alpha$-choice fair} if for all inputs $\delta >0$ and
$\alpha > 0$: for all $M$, all rounds $t$, all states
$s$ and actions $a, a'$:
\begin{equation*}
	\hspace{-1mm}	Q_M^*(s,a) \geq Q_M^*(s,a') \Rightarrow \mathcal{L}(s, a, \hist_{t-1}) \geq \mathcal{L}(s, a', \hist_{t-1}) - \alpha,
\end{equation*}
with probability of at least $1-\delta$ over histories $\hist_{t-1}$.
If \alg is $\alpha$-choice fair for any input $\alpha>0$, we call \alg
\emph{\pfair~fair}.
\end{definition}

A slight modification of the lower bound for (exact) fairness shows that
algorithms satisfying \pfair fairness can also require exponential
time to achieve near-optimality. We therefore propose an alternative
relaxation, where we relax the \emph{quality} requirement. As described in
Section~\ref{sec:related}, the resulting notion of \afair
fairness is in some sense the most fitting relaxation of fairness, and is
a particularly attractive one because it allows us to give
algorithms circumventing the exponential hardness proved for fairness and
\pfair fairness.

\begin{definition}[Approximate-action Fairness]
\label{def:fair-approx}
$\mathcal{L}$ is \emph{$\alpha$-action fair} if for all inputs
$\delta>0$ and $\alpha >0$, for all $M$, all rounds $t$, all states
$s$ and actions $a, a'$:
\begin{equation*}
  Q^*_M(s,a) > Q^*_M(s,a') + \alpha \Rightarrow \mathcal{L}(s, a, \hist_{t-1}) \geq \mathcal{L}(s, a' \hist_{t-1})
\end{equation*}
with probability of at least $1-\delta$ over histories $h_{t-1}$.
If \alg is $\alpha$-action fair for any input $\alpha>0$, we call \alg
\emph{\afair~fair}.
\end{definition}

Approximate-choice fairness prevents equally good actions from  being chosen at very different rates,
while approximate-action fairness prevents substantially worse actions from being chosen over better ones.
In hiring, an approximately-action fair firm can only (probabilistically) target one population over another if
the targeted population is not substantially worse. While this is a weaker guarantee, it at least forces
an approximately-action fair algorithm to learn different populations to statistical confidence. This is
a step forward from current practices, in which companies have much higher degrees of uncertainty about
the quality (and impact) of hiring individuals from under-represented populations. For this reason and the computational benefits mentioned above, 
our upper bounds will primarily focus on \afair fairness.

We now state several useful observations regarding fairness.  We defer
all the formal statements and their proofs to the Appendix.
We note that there always exists a (possibly randomized) optimal
policy which is fair (Observation~\ref{obs:exact-opt}); moreover,
\emph{any} optimal policy (deterministic or randomized) is \afair fair
(Observation~\ref{obs:approx-fair}), as is the uniformly random policy
(Observation~\ref{obs:random-walk}).

Finally, we consider a restriction of the actions in an MDP $M$ to
nearly-optimal actions (as measured by $Q^*_M$ values).
\begin{definition}[Restricted MDP]
\label{def:fair-mdp}
The \emph{$\alpha$-restricted} MDP of $M$, denoted by $M^\alpha$, is
identical to $M$ except that in each state $s$, the set of available
actions are restricted to
$\{a: Q^*_M(s, a) \geq \max_{a'\in \mathcal{A}_M}Q^*_M(s, a') - \alpha
\mid a \in \mathcal{A}_M\}$.
\end{definition}

$M^\alpha$ has the following two properties: (i) any policy in
$M^\alpha$ is $\alpha$-action fair in $M$
(Observation~\ref{obs:m-alphaa}) and (ii) the optimal policy in
$M^\alpha$ is also optimal in $M$ (Observation~\ref{obs:m-alpha-opt}).
Observations~\ref{obs:m-alphaa}~and~\ref{obs:m-alpha-opt} aid our
design of an \afair fair algorithm: we construct $M^\alpha$ from
estimates of the $Q^*_M$ values (see Section~\ref{sec:fair-planning}
for more details).

%% file: lower-bound.tex
\section{Lower Bounds}
\label{sec:lower-bound}
We now demonstrate a stark separation between the
performance of learning algorithms with and without
fairness.  First, we show that neither fair nor \pfair fair algorithms
achieve near-optimality unless
the number of time steps $\T$ is at least $\Omega(k^n)$, exponential
in the size of the state space.  We then show that any \afair fair
algorithm requires a number of time steps $\T$ that is at least
\ifnum\bug=1 $\Omega(k^{\tfrac{1}{\sqrt{1-\gamma}}})$ \else
$\Omega(k^{\tfrac{1}{1-\gamma}})$ \fi to achieve near-optimality.  We start
by proving a lower bound for fair algorithms.
\begin{theorem}
\label{thm:lower}
If $\delta < \tfrac{1}{4}$, $\gamma > \tfrac{1}{2} $ and
$\epsilon < \tfrac{1}{8}$, no fair algorithm can be $\epsilon$-optimal
in $\T = O(k^n)$ steps.\footnote{We have not optimized the constants
  upper-bounding parameters in the statement of
  Theorems~\ref{thm:lower},~\ref{thm:lower2} and~\ref{thm:lower1}. The values presented here are only chosen for
  convenience.}  \iftoggle{condense}{}{\jm{$O(k^n)}$}
\end{theorem}
Standard reinforcement learning algorithms (absent a fairness
constraint) learn an $\epsilon$-optimal policy in a number of steps
polynomial in $n$ and $\tfrac{1}{\epsilon}$; Theorem~\ref{thm:lower}
therefore shows a steep cost of imposing fairness.  We outline the
idea for proof of Theorem~\ref{thm:lower}.
For intuition, first consider the special case when the number of
actions $k=2$.  We introduce the MDPs witnessing the claim in
Theorem~\ref{thm:lower} for this case.

  \iftoggle{condense} {
\begin{definition}[Lower Bound Example]
\label{def:mdp-lb}
For $\mathcal{A}_M = \{L,R\}$, let
$M(x) = (\mathcal{S}_M, \mathcal{A}_M, \mathcal{P}_M, \mathcal{R}_M,
T, \gamma, x)$ be an MDP with
\begin{itemize}
	\setlength\itemsep{0.1em}
	\item for all $i \in [n]$, $P_M(s_i,L,s_1) = P_M(s_i,R,s_j) = 1$ where
$j = \min\{i+1, n\}$ and is 0 otherwise.
	\item for $i \in [n-1]$, $R_M(s_i) = 0.5$, and $R_M(s_n) = x$.
\end{itemize}
\end{definition}
}
{
\begin{definition}
\label{def:mdp-lb}
Let $M(x) = (\mathcal{S}_M, \mathcal{A}_M, P_M, R_M, T, \gamma, x)$ be
an MDP where
	\begin{itemize}
		\item $\mathcal{S}_M = \{s_1, \ldots, s_n\}$.
		\item $\mathcal{A}_M = \{L,R\}$.
		\item for all $i\in[n]$, $P_M(s_i,L,s_1) =
                  P_M(s_i,R,s_j) = 1$
                  where $j = \min\{i+1, n\}$ and is 0 otherwise
                  (transitions are deterministic).
		\item for $i \in [n-1]$, $R_M(s_i) = 0.5$, and
                  $R_M(s_n) = x$ (rewards are deterministic).
	\end{itemize}
\end{definition}
}

\begin{figure}[h]
\centering
\includegraphics[width=0.7\textwidth]{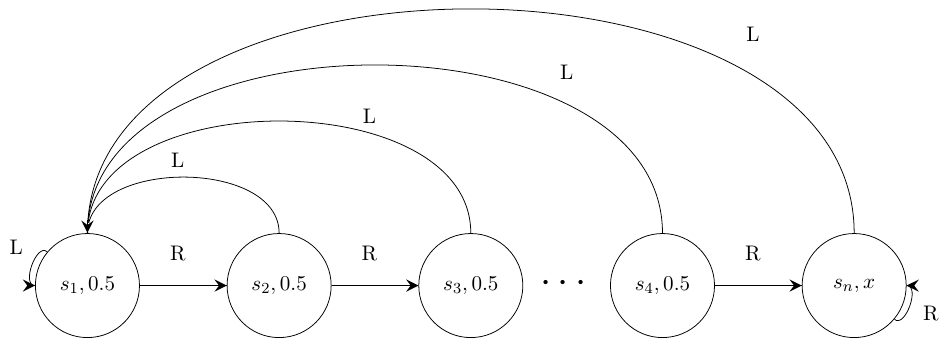}
  \caption{MDP$(x)$: Circles represent states (labels denote
    the state name and deterministic reward). Arrows
    represent actions.\label{fig:lb}}
\end{figure}

Figure~\ref{fig:lb} illustrates the MDP from
Definition~\ref{def:mdp-lb}.  All the transitions and rewards in $M$
are deterministic, but the reward at state $s_n$ can be either $1$ or
$\tfrac{1}{2}$, and so no algorithm (fair or otherwise) can determine
whether the $Q^*_M$ values of all the states are the same or not until
it reaches $s_n$ and observes its reward.  Until then, fairness
requires that the algorithm play all the actions uniformly at random
(if the reward at $s_n$ is $\tfrac{1}{2}$, any fair algorithm must
play uniformly at random forever).  Thus, \emph{any} fair algorithm
will take exponential time in the number of states to reach $s_n$.
This can be easily modified for $k> 2$: from each state $s_i$, $k-1$
of the actions from state $s_i$ (deterministically) return to state
$s_1$ and only one action (deterministically) reaches any other state
$s_{\min\{i+1, n\}}$.  It will take $k^n$ steps before any fair
algorithm reaches $s_n$ and can stop playing uniformly at random
(which is necessary for near-optimality).  The same example, with a
slightly modified analysis, also provides a lower bound of
$\Omega((k/(1+k\alpha))^n)$ time steps for approximate-choice fair
algorithms as stated in Theorem~\ref{thm:lower2}.
\begin{theorem}
\label{thm:lower2}
If
$\delta <\tfrac{1}{4}, \alpha < \tfrac{1}{4}, \gamma > \tfrac{1}{2}$
and $\epsilon < \tfrac{1}{8}$, no $\alpha$-choice fair algorithm is
$\epsilon$-optimal for $\T = O((\tfrac{k}{1+k\alpha})^n)$
steps.
 \end{theorem}
Fairness and \pfair~fairness
 are both extremely costly, ruling out polynomial time learning rates.  Hence, we
 focus on \afair fairness. Before moving to positive results, we
 mention that the time complexity of \afair fair algorithms will still
 suffer from an exponential dependence on $\tfrac{1}{1-\gamma}$.
\begin{theorem}
\label{thm:lower1}
For $\delta <\tfrac{1}{4}$, $\alpha < \tfrac{1}{8}$, $\gamma > \max(0.9,c)$,
$c \in (\tfrac{1}{2},1)$ and $\epsilon <\tfrac{1 - e^{c-1}}{16}$,
no $\alpha$-action fair algorithm is $\epsilon$-optimal
for $\T = O((k^{\tfrac{1}{1-\gamma}})^c)$ steps.
\end{theorem}
The MDP in Figure~\ref{fig:lb} also witnesses the claim of Theorem~\ref{thm:lower1} when
$n = \lceil \frac{\log(1/(2\alpha))}{1-\gamma}\rceil$.  The discount
factor $\gamma$ is generally taken as a constant, so in most
interesting cases $\tfrac{1}{1-\gamma} \ll n$: this lower bound is
substantially less stringent than the lower bounds proven for fairness
and \pfair fairness.  Hence, from now on, we focus
on designing algorithms satisfying \afair fairness with learning
rates polynomial in every parameter but $\tfrac{1}{1-\gamma}$,
\ifnum\bug=1 since we show that the learning rate of any \afair
algorithm has a super-polynomial dependence on $\tfrac{1}{1-\gamma}$.
\else and with tight dependence on $\tfrac{1}{1-\gamma}$.  \fi

%% file: upper-bound.tex
\section{A Fair and Efficient Learning Algorithm}
\label{sec:upper-bound}

We now present an \afair fair algorithm, \ecubep with
the performance guarantees stated below.

\begin{theorem}
  \label{thm:fair-rl} 
  Given $\epsilon > 0$, $\alpha > 0$,
  $\delta\in\left(0, \tfrac{1}{2}\right)$ and $\gamma \in [0, 1)$ as
  inputs, \ecubep is an $\alpha$-action fair algorithm which achieves
  $\epsilon$-optimality after
\begin{equation}
\label{eq:sample-complexity}
\T = \tilde O\left( \frac{n^5 \Te k^{\frac{1}{1-\gamma}+5}}{\min\{\alpha^4,
\epsilon^4\} \epsilon^2\left(1-\gamma\right)^{12}}\right)
\iffalse
\log\left(\frac{n}{\delta}
\log\left(\frac{k}{\delta}\right)\log\left(\frac{1}{\delta}\right)\log^9\left(\frac{1}{\epsilon\left(1-\gamma\right)}\right)\right)
\fi
\end{equation}
steps where $\tilde O$ hides poly-logarithmic terms.
\end{theorem}

The running time of \ecubep (which we have not attempted to optimize) is polynomial in all the parameters of the MDP  except
$\tfrac{1}{1-\gamma}$; Theorem~\ref{thm:lower1} implies that
\ifnum\bug=1 no \afair fair algorithm can have a running time that is
polynomial in $\tfrac{1}{1-\gamma}$.  \else this exponential dependence on
$\tfrac{1}{1-\gamma}$ is necessary.  \fi 

Several more recent
algorithms (e.g. R-MAX~\citen{BrafmanT02}) have improved upon the
performance of \ecube. We adapted \ecube~primarily for its
simplicity. While the machinery required to properly balance
fairness and performance is somewhat involved, the basic ideas of our
adaptation are intuitive. We further note that subsequent algorithms
improving on \ecube~tend to heavily leverage the principle of
``optimism in face of uncertainty": such behavior often violates
fairness, which generally requires \emph{uniformity} in the face of
uncertainty. Thus, adapting these algorithms to satisfy fairness is
more difficult.  This in particular suggests
\ecube~as an apt starting point for designing a fair planning
algorithm.

The remainder of this section will explain \ecubep, beginning with a high-level
description in Section~\ref{sec:high}. We then define the ``known" states
\ecubep uses to plan in Section~\ref{sec:known}, explain this planning process
in Section~\ref{sec:fair-planning}, and bring this all together to prove
\ecubep's fairness and performance guarantees in Section~\ref{sec:analysis}. 

\input{high-level}
\input{known}
\input{plan}
\input{analysis}

%% file: high-level.tex
\subsection{Informal Description of \ecubep}
\label{sec:high}

\ecubep relies on the notion of ``known'' states.  A state $s$
is defined to be \emph{known} after all actions have been chosen
from $s$ enough times to confidently estimate relevant reward
distributions, transition probabilities, and $Q^{\pi}_M$ values for each
action. At each time $t$, \ecubep then uses known states to reason about the 
MDP as follows:
\begin{itemize}
\item If in an unknown state, take a uniformly random trajectory of length
$\He$.
\item If in a known state, compute \emph{(i)} an exploration policy which escapes
to an unknown state quickly and $p$, the probability that this policy reaches
an unknown state within $2T^*_\epsilon$ steps, and \emph{(ii)} an exploitation
policy which is near-optimal in the known states of $M$.
\begin{itemize}
	\item If $p$ is large enough, follow the exploration policy; otherwise, follow
	the exploitation policy.
\end{itemize}
\end{itemize}

\ecubep thus relies on known states to balance exploration and exploitation in
a reliable way. 
While \ecubep and \ecube share this general idea, fairness forces \ecubep to
more delicately balance exploration and exploitation. For example, while
both algorithms explore until states become ``known", the definition of a known
state must be much stronger in \ecubep than in \ecube because \ecubep
additionally requires accurate estimates of actions' $Q^{\pi}_M$ values in
order to make decisions without violating fairness. For this reason, \ecubep
replaces the deterministic exploratory actions of \ecube with random
trajectories of actions from unknown states. These random trajectories are then
used to estimate the necessary $Q^\pi_M$ values.

In a similar vein, \ecubep requires particular care in computing exploration and exploitation
policies, and must restrict the set of such policies to fair exploration and fair exploitation
policies. Correctly formulating this restriction process to balance fairness
and performance relies heavily on the observations about the relationship
between fairness and performance provided in Section~\ref{sec:fairness}.

%% file: known.tex
\subsection{Known States in \ecubep}
\label{sec:known}
We now formally define the notion of known states for \ecubep. We say a state
$s$ becomes known when one can compute good estimates of
\emph{(i)} $R_M(s)$ and $P_M(s,a)$ for all $a$, and \emph{(ii)}
$Q^{*}_M(s,a)$ for all $a$.
\begin{definition}[Known State]
\label{def:known}
Let
\[
m_1=O\left(k^{\He+3}n\left(\frac{\rmax}{\left(1-\gamma\right)\alpha}\right)^2 \log\left(\frac{k}{\delta}\right)\right) \text{ and }
m_2=O\left(\left(\frac{n}{\min\{\epsilon, \alpha\}}\right)^4\He^8\log\left(\frac{1}{\delta}\right)\right).
\]
A state $s$ becomes \emph{known} after taking
\begin{equation}
m_{Q} := k \cdot\max\{m_1, m_2\}
\label{eq:Q-known}
\end{equation}
length-$\He$ random trajectories from $s$.
\end{definition}
It remains to show that motivating conditions \emph{(i)} and \emph{(ii)} indeed
hold for our formal definition of a known state.
Informally, $m_1$ random trajectories suffice to ensure that we
have accurate estimates of all $Q^*_M(s,a)$ values, and $m_2$ random
trajectories suffice to ensure accurate estimates of the transition
probabilities and rewards.

To formalize condition \emph{(i)}, we rely on Theorem~\ref{thm:kmn99},
connecting the number of random trajectories taken from $s$ to the accuracy
of the empirical $V^\pi_M$ estimates.

\begin{theorem}[Theorem 5.5,~\citet{KearnsMN99}]
\label{thm:kmn99}
 For any state $s$ and $\alpha > 0$, after
$$m = O\left(k^{\He+3}\left(\frac{\rmax}{\left(1-\gamma\right)\alpha}\right)^2\log\left(\frac{|\Pi|}{\delta}\right)\right)$$
random trajectories of length $\He$ from $s$, with probability of at least
$1-\delta$, we can compute estimates $\hat{V}_M^{\pi}$ such that 
$|V^\pi_M\left(s\right) - \hat{V}^\pi_M\left(s\right)|\leq \alpha$,
simultaneously for all $\pi\in \Pi$.
\end{theorem}

Theorem \ref{thm:kmn99} enables us to translate between the number of
trajectories taken from a state and the uncertainty about its
$V_M^\pi$ values for all policies (including $\pi^*$ and hence
$V^*_M$).  Since $|\Pi| = k^n$, we substitute
$\log\left(|\Pi|\right) = n\log\left(k\right)$.  To estimate
$Q^*_M\left(s,a\right)$ values using the $V^*_M\left(s\right)$ values
we increase the number of necessary length-$\He$ random trajectories
by a factor of $k$.

For condition \emph{(ii)}, we adapt the analysis of
\ecube\citen{KearnsS02}, which states that if each action in a state $s$
is taken $m_2$ times, then the transition probabilities and reward in
state $s$ can be estimated accurately (see Section~\ref{sec:analysis}).  

%% file: plan.tex
\subsection{Planning in \ecubep}
\label{sec:fair-planning}

We now formalize the planning steps in \ecubep from known states.
For the remainder of our exposition, we make
Assumption~\ref{assumpt:known} for convenience (and show how to remove
this assumption in the Appendix).
\begin{assumpt}
\label{assumpt:known}
$\Te$ is known.
\end{assumpt}

\input{known-fig}

\ecubep constructs two ancillary MDPs for planning: $\mknown$ is the 
\emph{exploitation} MDP, in which the unknown states of $M$ are condensed into a
single absorbing state $s_0$ with no reward.  In the known states $\known$,
transitions are kept intact and the rewards are deterministically set to their
mean value. $\mknown$ thus incentivizes exploitation by giving reward only for
staying within known states. In contrast, $\munknown$ is the \emph{exploration}
MDP, identical to $\mknown$ except for the rewards. The rewards in the known
states $\known$ are set to $0$ and the reward in $s_0$ is set to $\rmax$. 
$\munknown$ then incentivizes exploration by giving reward only for escaping to
unknown states. See the middle (right) panel of
Figure~\ref{fig:induced} for an illustration of $\mknown$ ($\munknown$), and
Appendix for formal definitions.

\ecubep uses these constructed MDPs to plan according to the following
natural idea: when in a known state, \ecubep constructs $\hmknown$ and
$\hmunknown$ based on the estimated transition and rewards observed so far (see
the Appendix for formal definitions), and then uses these to
compute additional restricted MDPs $\hmknown^\alpha$ and $\hmunknown^\alpha$ 
for \afair fairness. \ecubep then uses these restricted MDPs to choose between
exploration and exploitation.

More formally, if the optimal policy in $\hmunknown^\alpha$ escapes to the
absorbing state of $\mknown$ with high enough probability within $2\Te$ steps,
then \ecubep explores by following that policy. Otherwise, \ecubep exploits by
following the optimal policy in $\hmknown^\alpha$ for $\Te$ steps. While
following either of these policies, whenever \ecubep encounters an unknown
state, it stops following the policy and proceeds by taking a length-$\He$
random trajectory.

%% file: known-fig.tex
\begin{figure}
\centering
\includegraphics[width=0.9\textwidth]{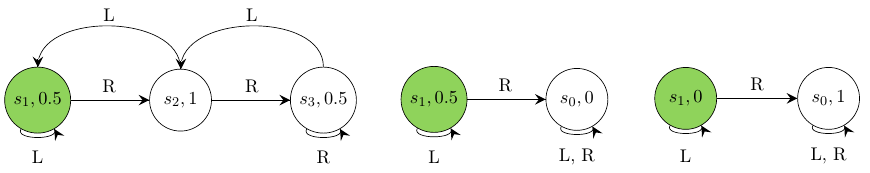}
\caption{Left: An MDP $M$ with two actions ($L$ and $R$) and deterministic transition functions and rewards. Green denotes the 
set of known states $\known$.
Middle: $\mknown$. Right: $\munknown$. \label{fig:induced}}
\end{figure}

%% file: analysis.tex
\subsection{Analysis of \ecubep}
\label{sec:analysis}

\input{analysis-background}

We now have the necessary results to prove Theorem~\ref{thm:fair-rl}.
\begin{proof}[Proof of Theorem~\ref{thm:fair-rl}]
  We divide the analysis into separate parts: the performance
  guarantee of \ecubep~and its \afair fairness.  We defer the analysis of the
  probability of failure of \ecubep to the
  Appendix.

\input{analysis-performance}

\input{analysis-fairness}

%% file: analysis-background.tex
In this section we formally analyze \ecubep~and prove
Theorem~\ref{thm:fair-rl}. We begin by proving that $\mknown^{\alpha}$
is useful in the following sense: $\mknown^{\alpha}$ has at least one
of an exploitation policy achieving high reward or an exploration
policy that quickly reaches an unknown state in $M$.  

\begin{lemma}[Exploit or Explore Lemma]
\label{lem:ex2}
For any state $s\in\known$, $\beta\in(0,1)$ and any $T> 0$ at least
one of the statements below holds:
\begin{itemize}
\item there exists an \emph{exploitation policy} $\pi$ in
  $\mknown^\alpha$ such that
\begin{align*}
\hspace{-1mm}\max_{\bar\pi\in\Pi}\E\sum_{t=1}^{T} V^{\bar\pi}_{M}\left(\bar\pi^t(s), T\right)
- \E\sum_{t=1}^{T} V^{\pi}_{\mknown}\left(\pi^t(s), T\right) \leq \beta T
\end{align*}
where the random variables $\pi^t(s)$ and $\bar\pi^t(s)$ denote the
states reached from $s$ after following $\pi$ and $\bar\pi$ for $t$
steps, respectively.
\item there exists an \emph{exploration policy} $\pi$ in $\mknown^\alpha$
  such that the probability that a walk of $2T$ steps from $s$ following
  $\pi$ will terminate in $s_0$ exceeds $\beta/T$.
\end{itemize}
\end{lemma}

We can use this fact to reason about exploration as follows.
First, since Observation~\ref{obs:approx-fair} tells us that the optimal policy
in $M$ is \afair fair, if the optimal policy stays in the set of $M$'s known
states $\mknown$, then following the optimal policy in $\mknown^{\alpha}$ is
both optimal and \afair fair.

However, if instead the optimal policy in $M$ quickly escapes to an unknown
state in $M$, the optimal policy in $\mknown^\alpha$ may not be able to compete
with the optimal policy in $M$.
Ignoring fairness, one natural way of computing an escape policy to ``keep up"
with the optimal policy is to compute the optimal policy in $\munknown$.
Unfortunately, following this escape policy might violate \afair
fairness -- high-quality actions might be ignored in lieu of low-quality
exploratory actions that quickly reach the unknown states of $M$.  Instead, we compute an escape policy in $\munknown^\alpha$
and show that if no near-optimal exploitation policy exists in $\mknown$, then the optimal
policy in $\munknown^\alpha$ (which is fair by construction) quickly escapes to
the unknown states of $M$.

Next, in order for \ecubep to check whether the optimal policy in
$\munknown^{\alpha}$ quickly reaches the absorbing state of $\mknown$ with 
significant probability, \ecubep simulates the execution of the optimal
policy of $\munknown^{\alpha}$ for $2\Te$ steps from the known state
$s$ in $\mknown^{\alpha}$ several times, counting the ratio of the runs ending in
$s_0$, and applying a Chernoff bound; this is where
Assumption~\ref{assumpt:known} is used.

Having discussed exploration, it remains to show that the exploitation policy
described in Lemma~\ref{lem:ex2} satisfies $\epsilon$-optimality as defined in Definition~\ref{def:performance}.
By setting $T \geq \Te$ in Lemma~\ref{lem:ex2} and applying
Lemmas~\ref{lem:mixing2}~and~\ref{lem:sim}, we can prove
Corollary~\ref{cor:exploit-policy} regarding this exploitation policy.
\begin{corollary}
\label{cor:exploit-policy}
For any state $s\in\known$ and $T\geq \Te$ if there
exists an exploitation policy $\pi$ in $\mknown^\alpha$ then
\[
\left|\frac{1}{T}\E\sum_{t=1}^{T} V^{\pi}_{M}\left(\pi^t(s), T\right)-\E_{s\sim\mu^*}V^*_M\left(s\right)\right|\leq \frac{\epsilon}{1-\gamma}.
\]
\end{corollary} 

Finally, we have so far elided the fact that \ecubep~only has access to the
\emph{empirically estimated} MDPs $\hmknown^\alpha$ and $\hmunknown^\alpha$
(see the Appendix for formal definitions). We remedy this
issue by showing that the behavior of any policy $\pi$ in $\hmknown^\alpha$
(and $\hmunknown^\alpha$) is similar to the behavior of $\pi$ in
$\mknown^\alpha$ (and $\munknown^\alpha$). To do so, we prove a stronger
claim: the behavior of  any $\pi$ in $\hmknown$ (and $\hmunknown)$ is
similar to the behavior of $\pi$ in $\mknown$ (and $\munknown)$.
%\mj{Last parenthetical seems redundant. I think we can move the precise
%statement of Lemma 10 to appendix and replace with the informal description above
%to save space without losing much.}

\begin{lemma}
\label{lem:sim}
Let $\known$ be the set of known states and $\hmknown$ the
approximation to $\mknown$.  Then for any state $s\in \known$, any
action $a$ and any policy $\pi$, with probability at least $1-\delta$:
\begin{enumerate}
\item
$
V^{\pi}_{\mknown}(s)  - \min\{\alpha/2, \epsilon\} \leq  V^{\pi}_{\hat \mknown}(s)\leq V^{\pi}_{\mknown}(s) + \min\{\alpha/2, \epsilon\},
$
\item
$
Q^{\pi}_{\mknown}\left(s, a\right)  - \min\{\alpha/2, \epsilon\} \leq  Q^{\pi}_{\hat \mknown}\left(s, a\right)\leq Q^{\pi}_{\mknown}\left(s, a\right) + \min\{\alpha/2, \epsilon\}.
$
\end{enumerate}
\end{lemma}

%% file: analysis-performance.tex
We start with the performance guarantee and show that when
\ecubep follows the exploitation policy the average $V^*_M$ values of
the visited states is close to
$\E_{s\sim\mu^*}V^*_M(s)$. However, when following an exploration
policy or taking random trajectories, visited states' $V^*_M$ values
can be small. To bound the performance of \ecubep,
we bound the number of these exploratory steps by the MDP parameters
so they only have a small effect on overall
performance.

Note that in each $\Te$-step exploitation phase of \ecubep, the
  expectation of the average $V^*_M$ values of the visited states is
  at least $\E_{s\sim\mu^*}V^*_M(s)-\epsilon/(1-\gamma)-\epsilon/2$ by
  Lemmas~\ref{lem:mixing2},~\ref{lem:ex2} and
  Observation~\ref{obs:m-alpha-opt}.  By a Chernoff bound, the
  probability that the actual average $V^*_M$ values of the visited
  states is less than
  $\E_{s\sim\mu^*}V^*_M(s)-\epsilon/(1-\gamma)-3\epsilon/4$ is less
  than $\delta/4$ if there are at least
  $\tfrac{\log(\tfrac{1}{\delta})}{\epsilon^2}$ exploitation phases.

We now bound the total number of exploratory steps of \ecubep~by
\[ T_1 = O\left(n m_Q\He + n m_Q\frac{\Te}{\epsilon}
  \log\left(\frac{n}{\delta}\right)\right),\]
 where $m_Q$ is defined in Equation~\ref{eq:Q-known} of
  Definition~\ref{def:known}. The two components of this term bound
  the number of rounds in which \ecubep plays non-exploitatively: the
  first bounds the number of steps taken   when \ecubep follows random trajectories,
and the second bounds how many steps are taken following
  explicit exploration policies. The former bound follows from the
  facts that each random trajectory has length $\He$; that in each
  state, $m_Q$ trajectories are sufficient for the state to become
  known; and that random trajectories are taken only before all
  $n$ states are known.  The latter bound follows from the fact that \ecubep~
follows an exploration policy for   $2\Te$ steps; and 
  an exploration policy needs to be followed only
  $O(\tfrac{\Te}{\epsilon}\log(\frac{n}{\delta}))$ times before reaching an
  unknown state (since any exploration policy will end up in an
  unknown state with probability of at least $\tfrac{\epsilon}{\Te}$ according
  to Lemma~\ref{lem:ex2}, and applying a Chernoff bound); that an
  unknown state becomes known after it is visited $m_Q$ times; and
  that exploration policies are only followed before all states are
  known.

Finally, to make up for the potentially poor performance in
exploration, the number of $2\Te$ steps exploitation phases needed is at
least
\[
T_2 =  O\left(\frac{T_1(1-\gamma)}{\epsilon}\right).
\]

Therefore, after $\T= T_1 + T_2$ steps we have
\[
\E_{s\sim \mu^*}V^*_M(s)  - \frac{1}{\T}\E\sum_{t=1}^{\T} V^*_M(s_t)  \leq \frac{2\epsilon}{1-\gamma},
\]
as claimed in Equation~\ref{eq:sample-complexity}.
%
%Note that $T$ in Equation~\ref{eq:t} is
%$\poly\left(n, k, 1/\epsilon, 1/\alpha, 1/(1-\gamma), \log(1/\delta),
%  k^{\He}, \Te\right)$.
The running time of \ecubep~is $O(\tfrac{n\T^3}{\epsilon})$: the additional
$\tfrac{nT^2}{\epsilon}$ factor comes from offline computation of the optimal
policies in $\hmknown^\alpha$ and $\hmunknown^\alpha.$ 

%% file: analysis-fairness.tex
\label{subsec:appx}
We wrap up by proving \ecubep satisfies \afair fairness in every
round.  The actions taken during random trajectories are fair (and
hence \afair fair) by Observation~\ref{obs:random-walk}.  Moreover,
\ecubep~computes policies in $\hmknown^{\alpha}$ and
$\hmunknown^{\alpha}$.  By Lemma~\ref{lem:sim} with probability at
least $1-\delta$ any $Q^*$ or $V^*$ value estimated in
$\hmknown^{\alpha}$ or $\hmunknown^{\alpha}$ is within $\alpha/2$ of
its corresponding true value in $\mknown^{\alpha}$ or
$\munknown^{\alpha}$. As a result, $\hmknown^{\alpha}$ and
$\hmunknown^{\alpha}$ \emph{(i)} contain all the optimal policies and
\emph{(ii)} only contain actions with $Q^*$ values within $\alpha$ of
the optimal actions.  It follows that any policy followed in
$\hmknown^{\alpha}$ and $\hmunknown^{\alpha}$ is $\alpha$-action fair,
so both the exploration and exploitation policies followed by
\ecubep~satisfy $\alpha$-action fairness, and \ecubep~is therefore
$\alpha$-action fair.
\end{proof} 

%% file: discussion.tex
\section{Discussion and Future Work}
\label{sec:future}
Our work leaves open several interesting questions. For example, we give 
an algorithm that has an undesirable exponential dependence on $1/(1-\gamma)$, but we 
\ifnum\bug=1
also show that no \afair fair algorithm can run in time polynomial in $1/(1-\gamma)$.
\else
show 
that this dependence is unavoidable for any \afair fair algorithm.
\fi
Without fairness, near-optimality in learning can be achieved in time that is polynomial 
in \emph{all} of the parameters of the underlying MDP. So, we can ask:
 does there exist a \emph{meaningful} fairness notion that enables reinforcement 
 learning in time polynomial in all parameters?
 
 Moreover, our fairness definitions remain open to further modulation. It 
remains unclear whether one can \emph{strengthen} our fairness guarantee to 
bind across time rather than simply across actions available at the moment
without large performance tradeoffs. Similarly, it is not obvious whether one
can gain performance by \emph{relaxing} the every-step nature of our fairness
guarantee in a way that still forbids discrimination. These and other
considerations suggest many questions for further study; we therefore position
our work as a first cut for incorporating fairness into a reinforcement
learning setting.

%% file: missing-proofs-prelim.tex
\subsection{Omitted Proofs for Section~\ref{sec:prelim}}
\label{sec:missing-proofs-prelim}

\begin{proof}[Proof of Lemma~\ref{lem:mixing2}]
Let $\hat\mu^{\pi}_T$  denote the 
distribution of $\pi$ on states of $M$ after following $\pi$ for
$T$ steps starting from $s$. Then we know
\begin{align*}
\E_{s\sim\mu^\pi}V^\pi_M(s)&\; - \frac{1}{T}\E\sum_{t=1}^T V^\pi_M(s_t) = \sum_{i=1}^n \left(\mu^{\pi}(s_i) - \hat{\mu}^{\pi}_T(s_i)\right)V^\pi_M(s_i) \\
\leq&\; \sum_{i=1}^n \left|\mu^{\pi}(s_i) - \hat{\mu}^{\pi}_T(s_i)\right|V^\pi_M(s_i) \leq  \frac{\epsilon}{1-\gamma}.
\end{align*}
The last inequality is due to the following observations: (i) $V^\pi_M(s_i)\leq \tfrac{1}{1-\gamma}$ as rewards are in $[0,1]$
and (ii) $\Sigma_{i=1}^n \left|\mu^{\pi}(s_i) - \hat{\mu}^{\pi}_T(s_i)\right|\leq \epsilon$ since $T$ is at least the $\epsilon$-mixing time of $\pi$.
\end{proof}

%% file: missing-proofs-lower-bound-new.tex
\subsection{Omitted Proofs for Section~\ref{sec:lower-bound}}
\label{sec:lower-bound-proofs}

We first state the following useful Lemma about $M$.
\begin{lemma}
\label{lem:v-star}
Let $M$ be the MDP in Definition~\ref{def:mdp-lb}. Then for any
$i\in\{1, \ldots, n\}$,
$V^*_M(s_i) < \tfrac{1 + 2\gamma^{n-i+1}}{2(1-\gamma)}$.
\end{lemma}
\begin{proof}
\begin{align*}
	V_M^*(s_i) =&\;
	\text{discounted reward before reaching state $n$} + \text{discounted reward from staying at state $n$} \\
	<&\; \left[\sum_{t=1}^{n-i-1} \frac{\gamma^t}{2}\right] + \frac{\gamma^{n-i+1}}
	{1-\gamma}= \left[\frac{1}{2}\left(\frac{1}{1-\gamma} - \frac{\gamma^{n-i}}
	{1-\gamma}\right)\right] + \frac{\gamma^{n-i+1}}{1-\gamma} = \frac{1-\gamma^{n-i}}{2(1-\gamma)} + \frac{\gamma^{n-i+1}}{1-\gamma} \\
	=&\; \frac{1+\gamma^{n-i}(2\gamma-1)}{2(1-\gamma)} < \frac{1+2\gamma^{n-i+1}}{2(1-\gamma)},
\end{align*}
via two applications of the summation formula for geometric series.
\end{proof}

\input{lower-exact}
\input{lower-choice}
\input{lower-action}

%% file: lower-exact.tex
\begin{proof}[Proof of Theorem~\ref{thm:lower}]
  We prove Theorem~\ref{thm:lower} for the special case of $k=2$
  first.  Consider coupling the run of a fair algorithm $\alg$ on both
  $M(0.5)$ and $M(1)$.  To achieve this, we can fix the randomness of
  $\alg$ up front, and use the same randomness on both MDPs.  The set
  of observations and hence the actions taken on both MDPs are
  identical until $\alg$ reaches state $s_n$.  Until then, with
  probability at least $1-\delta$, $\alg$ must play $L$ and $R$ with
  equal probability in order to satisfy fairness (since, for $M(0.5)$,
  the only fair policy is to play both actions with equal probability
  at each time step). We will upper-bound the optimality of uniform
  play and lower-bound the number of rounds before which $s_n$ is
  visited by uniformly random play.

  Let $f_\gamma = \lceil\fgamma\rceil$ and $\T = 2^{n-2f_\gamma}$ for
  $n \geq 100(f_\gamma)^2$. First observe that the probability of
  reaching a fixed state $s_i$ for any $i\geq n-f_\gamma$ from a
  random walk of length $\T$ is upper bounded by the probability that
  the random walk takes $i \geq n-f_\gamma$ consecutive steps to the
  right in the first $\T$ steps. This probability is at most
  $p = 2^{n-2f_\gamma}(\tfrac{1}{2})^{n-f_\gamma}=2^{-f_\gamma}$ for
  any fixed $i$. Since reaching any state $i > i'$ requires reaching
  state $i'$, the probability that the $\T$ step random walk arrives
  in any state $s_i$ for $i\geq n-f_\gamma$ is also upper bounded by
  $p$.
 
  Next, we observe that $V^*_M(s_i)$ is a nondecreasing function
  of $i$ for both MDPs. Then the average $V^*_M$ values of the visited
  states of \emph{any} fair policy can be broken into two pieces: the
  average conditioned on (the  probability at least $1-\delta$ event) that the
  algorithm plays uniformly at random before reaching state $s_n$
  \emph{and} never reaching a state beyond $s_{n-f_\gamma}$, and the
  average conditioned on (the probability at most $\delta$ event) that
  the algorithm does not make uniformly random choices \emph{or} the
  uniform random walk of length $\T$ reaches a state beyond
  $s_{n-f_\gamma}$. So, we have that
  \begin{align*}
  	\frac{1}{\T}\mathbb{E}\sum_{t=1}^\T V_M^*(s_t) &\leq
  	\left(1-p-\delta\right) V^*_M(s_{n-f_\gamma})+ \left(p+\delta\right)\frac{1}{1-\gamma}\\
	&\leq \left(1-p-\delta\right) \frac{1+2\gamma^{f_\gamma+1}}{2(1-\gamma)}+ \left(p+\delta\right)\frac{1}{1-\gamma}.
  \end{align*}

  The first inequality follows from the fact that
  $V^*_M(s_i)\leq \tfrac{1}{1-\gamma}$ for all $i$, and the second from
  Lemma~\ref{lem:v-star} along with $V^*_M$ values being nondecreasing
  in $i$.  Putting it all together,
 \begin{align*}
	\E_{s\sim \mu^*}V^*_M(s) - \frac{1}{\T}\mathbb{E}\sum_{t=1}^\T V_M^*(s_t) \geq&\; \frac{1}{1-\gamma}-\left[\left(1-p-\delta\right) \frac{1+2\gamma^{f_\gamma+1}}{2(1-\gamma)}+ \left(p+\delta\right)\frac{1}{1-\gamma}\right] \\
	=&\;  \frac{1-p-\delta}{1-\gamma}\left[1-\frac{1+2\gamma^{f_\gamma+1}}{2}\right].
\end{align*}
So $\epsilon$-optimality requires
\begin{equation}
\label{eq:opt}
 \frac{2\epsilon}{1-\gamma} \geq \frac{1-p-\delta}{1-\gamma}\left[1-\frac{1+2\gamma^{f_\gamma+1}}{2}\right].
\end{equation}
However, if $\epsilon < \tfrac{1}{8}$ we get 
\begin{align*}
	\frac{2\epsilon}{1-\gamma} < &\; \frac{1- 0.04 - 1/4}{1-\gamma}\left[1 - \frac{1 + 2\times e^{-3}}{2}\right] < \frac{1-2^{-f_\gamma} - \delta}{1-\gamma} \left[1 - \frac{1+2\gamma^{f_\gamma+1}}{2}\right],
\end{align*}
where the third inequality follows when $\delta < \tfrac{1}{4}$ and $\gamma > \frac{1}{2} $. This means $\epsilon < \tfrac{1}{8}$ makes $\epsilon$-optimality impossible, as desired.
 
 Throughout we considered the special case of $k=2$ and proved a lower
bound of $\Omega(2^n)$ time steps for any fair algorithm satisfying
the $\epsilon$-optimality condition.  However, it is easy to see that MDP
$M$ in Definition~\ref{def:mdp-lb} can be easily modified in a way
that $k-1$ of the actions from state $s_i$ reach state $s_1$ and only
one action in each state $s_i$ reaches states $s_{\min\{i+1, n\}}$.
Hence, a lower bound of $\Omega(k^n)$ time steps can be similarly
proved.
\end{proof}

%% file: lower-choice.tex
\begin{proof}[Proof of Theorem~\ref{thm:lower2}]
  We mimic the argument used to prove Theorem~\ref{thm:lower} with the
  difference that, until visiting $s_n$, $\alg$ may not play $R$ with
  probability more than $\tfrac{1}{2} + \alpha$ (as opposed to $\tfrac{1}{2}$
  in Theorem~\ref{thm:lower}).
 Let $f_\gamma = \lceil\fgamma\rceil$ and $\T = (\tfrac{2}{1+2\alpha})^{n-2f_\gamma}$ for $n\geq 100 (f_\gamma)^2$. By a similar process
 as in Theorem~\ref{thm:lower}, the probability of reaching 
 state $s_i$ for any $i\geq n-f_\gamma$ from a random walk of length $\T$ is
 bounded by $p = (\tfrac{2}{1+2\alpha})^{-f_\gamma}$, and so the probability that
 the $\T$ steps random walk arrives in any state $s_i$ for $i\geq n-f_\gamma$
 is bounded by $p$. Carrying out the same process used to prove
 Theorem~\ref{thm:lower} then once more implies that $\epsilon$-optimality
 requires Equation~\ref{eq:opt} to hold
 when $\delta < \tfrac{1}{4}$, $\alpha < \tfrac{1}{4}$ and $\gamma > \tfrac{1}{2}$. Hence, $\epsilon < \tfrac{1}{8}$ violates this condition as desired.
 
 Finally, throughout we considered the special case of $k=2$. The same trick
 as in the proof of Theorem~\ref{thm:lower} can be used to prove the lower
 bound of $\Omega((\tfrac{k}{1+k\alpha})^n)$ time steps for any fair algorithm
 satisfying the $\epsilon$-optimality condition.
\end{proof}

%% file: lower-action.tex
\begin{proof}[Proof of Theorem~\ref{thm:lower1}]
 We also prove Theorem~\ref{thm:lower1} for the special case of $k=2$
 first, again considering the MDP in
  Definition~\ref{def:mdp-lb}.
We set the size of the state space in $M$ to be $n=\lceil \tfrac{\log(\tfrac{1}{2\alpha})}{1-\gamma} \rceil$.
  Then given the parameter ranges,
   for any $i$, $Q^*_M(s_i, R) - Q^*_M(s_i, L) > \alpha$ in M(1). 
  Therefore, any \afair fair algorithm should play actions R and L with 
  equal probability.

  Let $\T = 2^{cn}=\Omega((2^{1/(1-\gamma)})^c)$. First observe that the
  probability of reaching a fixed state $s_i$ for any $i\geq (c+1)n/2$
  from a random walk of length $\T$ is upper bounded by the
  probability that the random walk takes $i \geq (c+1)n/2$ consecutive
  steps to the right in the first $\T$ steps. This probability is at
  most $p = 2^{cn}2^{-(c+1)n/2} = 2^{(c-1)n/2}$ for any fixed
  $i$. Then the probability that the $\T$ steps random walk arrives in
  any state $s_i$ for $i \geq (c+1)n/2$ is also upper bounded by $p$.
 
  Next, we observe that $V^*_M(s_i)$ is a nondecreasing function
  of $i$, for both MDPs. Then the average $V^*_M$ values of the
  visited states of \emph{any} fair policy can be broken into two
  pieces: the average conditioned on the $1-\delta$ fairness
  \emph{and} never reaching a state beyond $s_{(c+1)n/2}$, and the
  average when fairness might be violated \emph{or} the uniform random
  walk of length $\T$ reaches a state beyond $s_{(c+1)n/2}$. So, we have
  that
  \begin{align*}
  	\frac{1}{\T}\mathbb{E}\sum_{t=1}^\T V_M^*(s_t) \leq&\;
  	\left(1-p-\delta\right) V^*_M(s_{(c+1)n/2}) + \left(p+\delta\right)\frac{1}{1-\gamma} \\
	\leq&\; \left(1-p-\delta\right) \frac{1+(2\gamma-1)\gamma^{\frac{(1-c)n}{2}}}{2(1-\gamma)} = \left(p+\delta\right)\frac{1}{1-\gamma}.
  \end{align*}

  The first inequality follows from the fact that
  $V^*_M(s_i)\leq \tfrac{1}{1-\gamma}$ for all $i$, and the second from (the line before the last in)
  Lemma~\ref{lem:v-star} along with $V^*_M$ values being nondecreasing
  in $i$.  Putting it all together,
 \begin{align*}
	\E_{s\sim \mu^*}&\;V^*_M(s)-\frac{1}{\T}\mathbb{E}\sum_{t=1}^\T V_M^*(s_t) \geq \frac{1}{1-\gamma} - \left(1-p-\delta\right) \frac{1+(2\gamma-1)
	\gamma^{\frac{(1-c)n}{2}}}{2(1-\gamma)} - \left(p+\delta\right)\frac{1}{1-\gamma} \\
	=&\;  \frac{1-p-\delta}{1-\gamma}\left[1-\frac{1+(2\gamma-1)\gamma^{\frac{(1-c)n}{2}}}{2}\right] =\frac{1-p-\delta}{1-\gamma}\left[\frac{1}{2}-\frac{(2\gamma-1)\gamma^{\frac{(1-c)n}{2}}}{2}\right] .
\end{align*}
So $\epsilon$-optimality requires
\[
	\frac{2\epsilon}{1-\gamma} \geq \frac{1-p-\delta}{1-\gamma}\left[\frac{1}{2}-\frac{(2\gamma-1)\gamma^{\frac{(1-c)n}{2}}}{2}\right].
\]
Rearranging and using $\delta < \tfrac{1}{4}$, we get that $\epsilon$-optimality requires
\[
	4\epsilon \geq \left[0.75 - 2^{\frac{(c-1)n}{2}}\right]\left[1 - (2\gamma-1)\gamma^{\frac{(1-c)n}{2}}\right]
\]
and expand $n$ to get
\begin{align*}
	\epsilon \geq&\; \frac{1}{4} \left[0.75 - 2^{\frac{(c-1)\log(\frac{1}{2\alpha})}{2(1-\gamma)}}\right] \times \left [1 - (2\gamma-1)\gamma^{\frac{(1-c)\log(\frac{1}{2\alpha})}{2(1-\gamma)}}\right]
	\equiv \frac{x y}{4}.
\end{align*}
Noting that $x$ is minimized when $2^{\tfrac{(c-1)\log(\frac{1}{2\alpha})}{2(1-\gamma)}}$ is maximized, and that this quantity is maximized when $\tfrac{\log(\frac{1}{2\alpha})}{2(1-\gamma)}$ is minimized (as $c-1$ is negative), we get that $\epsilon$-optimality requires
\[
	\epsilon \geq \frac{ \left[0.75 - 2^{\frac{c-1}{1-\gamma}}\right] y}{4}
\]
from $\alpha < \tfrac{1}{8}$. Similarly, $\alpha < \tfrac{1}{8}$ implies that $\epsilon$-optimality requires
\[
	\epsilon \geq \frac{\left[0.75 - 2^{\frac{c-1}{1-\gamma}}\right] \left [1 - (2\gamma -1)\gamma^{\frac{1-c}{1-\gamma}}\right]}{4}.
\]
Note that $0.75 - 2^{\frac{c-1}{1-\gamma}}$ is minimized when $\gamma$ is small, so $\gamma > c$ implies that $\epsilon$-optimality requires
\begin{align*}
	\epsilon \geq&\; \frac{ \left[0.75 - 2^{-1}\right] \left[1 - (2\gamma -1)\gamma^{\frac{1-c}{1-\gamma}}\right]}{4} \geq \frac{1}{16} \left[1 - (2\gamma -1)\gamma^{\frac{1-c}{2(1-\gamma)}}\right].
\end{align*}
Conversely, $1 - (2\gamma -1)\gamma^{\frac{1-c}{1-\gamma}}$ is minimized when $\gamma$ is large, so as 
\[
	\lim_{\gamma \to 1} \left(2\gamma -1\right)\gamma^{\frac{1-c}{1-\gamma}} = e^{c-1}
\]
we get that $\epsilon$-optimality requires
\[
	\epsilon \geq \frac{1}{16}\left(1 - e^{c-1}\right).
\]
 
  Finally, the same trick as in the proof of
  Theorem~\ref{thm:lower} can be used to prove the
  $\Omega((k^{1/(1-\gamma)})^c)$ lower bound for $k>2$ actions.
\end{proof}

%% file: missing-proofs-fair-ecube.tex
\subsection{Omitted Proofs for Section~\ref{sec:upper-bound}}
\label{sec:missing-proofs-ecube}
\begin{proof}[Proof of Lemma~\ref{lem:ex2}]
We first show that either
\begin{itemize}
\item
there exists
an \emph{exploitation policy} $\pi$ in $\mknown$ such that
$$
\frac{1}{T}\max_{\bar\pi\in\Pi}\E\sum_{t=1}^{T} V^{\bar\pi}_{M}\left(\bar\pi^t(s), T\right)  - \frac{1}{T}\E\sum_{t=1}^{T} V^{\pi}_{\mknown}\left(\pi^t(s), T\right)  \leq \beta
$$
where the random variables $\pi^t(s)$ and $\bar\pi^t(s)$ denote the states
reached from $s$ after following $\pi$ and $\bar\pi$ for $t$ steps,
respectively, or
\item there exists an \emph{exploration policy} $\pi$ in $\mknown$
  such that the probability that a walk of $2T$ steps from $s$ following
  $\pi$ will terminate in $s_0$ exceeds $\tfrac{\beta}{T}$.
\end{itemize}

  Let $\pi$ be a policy in $M$ satisfying
  \[ \frac{1}{T}\E\sum_{t=1}^T V^\pi_M(\pi^t(s), T)
  =\frac{1}{T}\max_{\bar\pi\in\Pi}\E\sum_{t=1}^{T}
  V^{\pi'}_{M}(\bar\pi^t(s), T) := \tilde{V}.
\]
For any state $s'$, let $p(s')$ denote all the paths of length $T$ in $M$ that
start in $s'$, $q(s')$ denote all the paths of length $T$ in $M$ that start in
$s'$ such that all the states in every path of length $T$ in $q(s')$
are in $\known$ and $r(s')$ all the paths of length $T$ in $M$ that start in
$s'$ such that at least one state in every path of length $T$ in $r(s')$
is not in $\known$.  Suppose
$$ \frac{1}{T}\E\sum_{t=1}^{T} V^{\pi}_{\mknown}(\pi^t(s)) <
\tilde{V} - \beta.$$
Otherwise, $\pi$ already witnesses the claim.  We show that a walk
of $2T$ steps from $s$ following $\pi$ will terminate in $s_0$ with
probability of at least $\tfrac{\beta}{T}$.
First,
\begin{align*}
\E\sum_{t=1}^{T} V^{\pi}_M(\pi^t(s), T) =&\; E\sum_{t=1}^{T}
\sum_{p(\pi^t(s))} \Pr[p(\pi^t(s))] V_M(p(\pi^t(s))) \\
=&\; \E\sum_{t=1}^{T} \sum_{q(\pi^t(s))} \Pr[q(\pi^t(s))] V_M(q(\pi^t(s))) + \E\sum_{t=1}^{T} \sum_{r(\pi^t(s))} \Pr[r(\pi^t(s))] V_M(r(\pi^t(s)))
\end{align*}

since $p(\pi^t(s)) = q(\pi^t(s)) \cup r(\pi^t(s))$, which is a disjoint union.
Next,
\begin{align*}
\E\sum_{t=1}^{T}\sum_{q(\pi^t(s))} \Pr[q(\pi^t(s))] V_M(q(\pi^t(s))) = &\;\E\sum_{t=1}^{T}\sum_{q(\pi^t(s))} \Pr^{\pi}_{\mknown}[q(\pi^t(s))] V_{\mknown}(q(\pi^t(s))) \\
\leq&\; \E\sum_{t=1}^{T}V^{\pi}_{\mknown}(\pi^t(s), T),
\end{align*}

where the equality is due to Definition~\ref{def:exploit-induced-mdp} and the
definition of $q$, and the inequality follows because $V^{\pi}_{\mknown}
(\pi^t(s), T)$ is the sum over all the $T$-paths in $\mknown$, not just those
that avoid the absorbing state $s_0$. Therefore by our original assumption on
$\pi$,
\begin{align*}
\E&\;\sum_{t=1}^{T}\sum_{q(\pi^t(s))} \Pr[q(\pi^t(s))] V_M(q(\pi^t(s))) \leq  \E\sum_{t=1}^{T}V^{\pi}_{\mknown}(\pi^t(s), T) <T\tilde{V} - T\beta.
\end{align*}

This implies
\begin{align*}
\E\sum_{t=1}^{T}\sum_{r(\pi^t(s))} \Pr[r(\pi^t(s))] V_M(r(\pi^t(s))) =&\; \E\sum_{t=1}^{T} V^{\pi}_M(\pi^t(s), T) -\; \E\sum_{t=1}^T\sum_{q(\pi^t(s))} \Pr[q(\pi^t(s))] V_M(q(\pi^t(s))) \\
=&\; T\tilde{V} - \E\sum_{t=1}^{T}\sum_{q(\pi^t(s))}\Pr[q(\pi^t(s))] V_M(q(\pi^t(s))) \geq T\beta,
\end{align*}

where the last step is the result of applying the previous inequality.
However,
\begin{align*}
\E&\;\sum_{t=1}^{T}\sum_{r(\pi^t(s))} \Pr[r(\pi^t(s))] V_M(r(\pi^t(s))) \leq T \E\sum_{t=1}^{T}\sum_{r(\pi^t(s))} \Pr[r(\pi^t(s))],
\end{align*}

because it is immediate that $V_M(r(\pi^t(s)))\leq T$ for all $\pi^t(s)$.  So
$T\beta \leq T \E\sum_{t=1}^T\sum_{r(\pi^t(s))}
\Pr[r(\pi^t(s))]$. % \leq T^2\rmax \Pr[r(s)]$.
Finally, if we let $\Pr_{2T}^{\pi}$
denote the probability that a walk of $2T$
steps following $\pi$
terminates in $s_0$,
i.e. the probability that $\pi$
escapes to an unknown state within $2T$
steps, then for each $t
\in [T]$, $\E\sum_{r(\pi^t(s))} \leq T\Pr_{2T}^{\pi}$. It follows that
\[
T\beta \leq T^2\Pr_{2T}^{\pi}
\]
and rearranging yields $\Pr_{2T}^{\pi} \geq \tfrac{\beta}{T}$ as desired.

Next, 
note that the exploitation policy (if it exists) can be derived by computing the
optimal policy in $\mknown$.
Moreover, the exploration policy (if it exists) in the
exploitation MDP $\mknown$ can indeed be derived by computing the
optimal policy in the exploration MDP $\munknown$ as observed by~\cite{KearnsS02}. 
Finally, by Observation~\ref{obs:m-alpha-opt}, any optimal
policy in $\hmknown^\alpha$ ($\hmunknown^\alpha$) is an optimal policy
in $\hmknown$ ($\hmunknown$)
\end{proof}

To prove Lemma~\ref{lem:sim}, we need some 
useful background adapted from~\citet{KearnsS02}.
\begin{definition}[Definition 7, \citet{KearnsS02}]
\label{def:mdp-approx}
Let $M$ and $\hat{M}$ be two MDPs 
with the same set of states and actions. We say
$\hat{M}$ is a $\beta$-approximation of $M$ if
\begin{itemize}
\item For any state $s$,
\[
\bar{R}_M(s)-\beta\leq \bar{R}_{\hat M}(s) \leq \bar{R}_M(s) +\beta.
\]
\item For any states $s$ and $s'$ and action $a$,
\[
P_M(s, a, s')-\beta \leq P_{\hat M}(s, a, s')\leq P_M(s, a, s') + \beta.
\]
\end{itemize}
\end{definition}
 
 \begin{lemma}[Lemma 5,~\citet{KearnsS02}]
\label{lem:packnown}
Let $M$ be an MDP and $\known$ the set of known states of $M$. 
For any $s, s'\in \known$ and action $a\in A$, 
let $\hat P_M(s,a,s')$ denote the empirical probability transition
estimates obtained from the visits to $s$. 
Moreover, for any state $s\in\known$ let $\bar{\hat R}(s)$ denote the empirical
estimates of the average reward obtained from visits to s.
Then with probability at least $1 - \delta$,
\begin{equation*}
	|\hat P_M(s, a, s') - P_M(s, a, s')| = O\left(\frac{\min\{\epsilon, \alpha\}^2}{n^2 \He^4}\right),
\end{equation*}
and
\begin{equation*}
	|\bar{\hat R}_M(s) - \bar{R}_M(s)| = O\left(\frac{\min\{\epsilon, \alpha\}^2}{n^2 \He^4}\right).
\end{equation*}
\end{lemma}

Lemma~\ref{lem:packnown} shows that $\hmknown$ and $\hmunknown$
are $O(\tfrac{\min\{\epsilon, \alpha\}^2}{n^2 \He^4})$-approximation MDPs for $\mknown$ and $\munknown$, respectively.

\begin{lemma}[Lemma 4,~\citet{KearnsS02}]
\label{lem:sim3}
Let $M$ be an MDP and $\hat{M}$ its $O(\tfrac{\min\{\epsilon, \alpha\}^2}{n^2 \He^4})$-approximation.
Then for any policy $\pi \in \Pi$ and any state $s$ and action $a$
\begin{equation*}
	V_{M}^{\pi}(s) - \min\{\epsilon, \alpha\} \leq V_{\hat{M}}^{\pi}(s) \leq V_{M}^{\pi}(s) + \min\{\epsilon, \frac{\alpha}{4}\},
\end{equation*}
and 
\begin{align*}
Q^{\pi}_{M}(s, a)  - \min\{\frac{\alpha}{4}, \epsilon\} \leq&\;  Q^{\pi}_{\hat{M}}(s, a) \leq Q^{\pi}_{M}(s, a) + \min\{\frac{\alpha}{4}, \epsilon\}.
\end{align*}
\end{lemma}

\begin{proof}[Proof of Lemma~\ref{lem:sim}]
  By Definition~\ref{def:known}~and~Lemma \ref{lem:packnown},   
  $\hmknown$ is a $O(\tfrac{\min\{\epsilon, \alpha\}^2}{n^2 \He^4})$-approximation of $\mknown$. 
  Then the statement directly follows by applying
  Lemma~\ref{lem:sim3}.
\end{proof}

\begin{proof}[Rest of the Proof of Theorem~\ref{thm:fair-rl}]
The only remaining part of the proof of Theorem~\ref{thm:fair-rl} is the analysis of the probability of failure of \ecubep.
To do so, we break down the probability of failure of \ecubep
by considering the following (exhaustive) list of possible failures:
\begin{enumerate}
\item At some known state the algorithm has a poor approximation of
  the next step, causing $\hmknown$ to not be a $O(\tfrac{\min\{\epsilon, \alpha\}^2}{n^2 \He^4})$-approximation of $\mknown$.
\item At some known state the algorithm has a poor approximation of
  the $Q^*_M$ values for one of the actions.
\item Following the exploration policy for $2\Te$ steps fails to yield enough
  visits to unknown states.
\item At some known state, the approximation value of that state in
  $\hmknown$ is not an accurate estimate for the value of the
  state in $\mknown$.
\end{enumerate}

We allocate $\tfrac{\delta}{4}$ of our total probability of failure to each of these sources:

\begin{enumerate}
\item Set $\delta' = \tfrac{\delta}{4n}$ in Lemma~\ref{lem:sim}.
\item Set $\delta' = \tfrac{\delta}{4nk}$ in Theorem~\ref{thm:kmn99}.
\item By Lemma~\ref{lem:ex2}, each attempted exploration is a Bernoulli trial with 
probability of success of at least $\tfrac{\epsilon}{4 \Te}$. In the worst case
we might need to make every state known before exploiting, 
leading to the $n m_Q$ trajectories ($m_Q$ as Equation~\ref{eq:Q-known} in Definition~\ref{def:known}) of length $\He$.
Therefore, the probability of taking fewer than $n m_Q$ trajectories of length $\He$ would be bounded by $\tfrac{\delta}{4}$ if the
number of $2\Te$ steps explorations is at least
\begin{equation}
\label{eq:m-exp}
m_\text{exp}= O\left(\frac{\Te n m_Q}{\epsilon} \log\left(\frac{n}{\delta}\right)\right).
\end{equation}
\item Set $\delta' = \tfrac{\delta}{4m_{\text{exp}}}$ ($m_{\text{exp}}$ as defined in 
Equation~\ref{eq:m-exp}) in Lemma~\ref{lem:sim}, as \ecubep~might
make $2\Te$ steps explorations up to $m_\text{exp}$ times.
\end{enumerate}
\end{proof}

\input{relaxation}

%% file: relaxation.tex
\subsection{Relaxing Assumption~\ref{assumpt:known}}
\label{sec:relax}
%\jm{I think this is ``supplementary.'' Should go into the appendix.}
%
Throughout Sections~\ref{sec:fair-planning} and \ref{sec:analysis} we
assumed that $\Te$, the $\epsilon$-mixing time of the optimal policy
$\pi^*$, was known (see Assumption~\ref{assumpt:known}).  Although
\ecubep~uses the knowledge of $\Te$ to decide whether to follow the
exploration or exploitation policy, Lemma~\ref{lem:ex2} continues to
hold even without this assumption.  Note that \ecubep~is parameterized
by $\Te$ and for any input $\Te$ runs in time $\poly(\Te)$. Thus if
$\Te$ is unknown, we can simply run \ecubep for $\Te=1, 2, \ldots$
sequentially and the running time and sample complexity will still be
$\poly(\Te)$. Similar to the analysis of \ecubep when $\Te$ is known
we have to run the new algorithm for sufficiently many steps so that
the possibly low $V^*_M$ values of the visited states in the early
stages are dominated by the near-optimal $V^*_M$ values of the visited
states for large enough guessed values of $\Te$.
%
%\mj{Should we condense this paragraph and make it a footnote? I don't really
%see why it's important enough to be the second to last paragraph.}
%\sj{too long to be a footnote so I moved it to a new subsection so it does 
%not interfere with the conclusion}

%% file: fair-observations.tex
\section{Observations on Optimality and Fairness}\label{sec:fair-obs}

\begin{obs}
\label{obs:exact-opt}
For any MDP $M$, there exists an optimal policy $\pi^*$ such that
$\pi^*$ is fair.
\end{obs}

\begin{proof}
  In time $t$, let state $s_t$ denote the state from which $\pi$
  chooses an action.  Let $a^* = \argmax_a Q_M^*(s_t, a)$ and
  $A^*(s_t) = \{a \in A \mid Q_M^*(s_t, a) = Q_M^*(s_t, a^*)\}$.  The
  policy of playing an action uniformly at random from $A^*(s_t)$ in
  state $s_t$ for all $t$, is fair and optimal.
\end{proof}

Approximate-action fairness, conversely, can be satisfied by \emph{any}
optimal policy, even a deterministic one.

\begin{obs}
\label{obs:approx-fair}
Let $\pi^*$ be an optimal policy in MDP $M$. Then $\pi^*$ is
\afair fair.
\end{obs}
\begin{proof}
  Assume that $\pi^*$ is not \afair fair.  Given state $s$, the
  action that $\pi^*$ takes from $s$ is uniquely determined since
  $\pi^*$ is deterministic we may denote it by $a^*$.  Then there
  exists a time step in which $\pi^*$ is in state $s$ and chooses
  action $a^*(s)$ such that there exists another action $a$ with
\[
	Q^*_M(s, a) > Q^*_M(s, a^*(s)) + \alpha,
\] 
a contradiction of the optimality of $\pi^*$. 
\end{proof}

Observations~\ref{obs:exact-opt}~and~\ref{obs:approx-fair} state that
policies with optimal performance are fair; we now state that playing
an action uniformly at random is also fair. 
\begin{obs}
\label{obs:random-walk}
An algorithm that, in every state, plays each action uniformly at
random (regardless of the history) is fair.
\end{obs}
\begin{proof}
  Let $\alg$ denote an algorithm that in every state plays uniformly
  at random between all available actions. Then
  $\alg(s, \hist_{t-1})_a = \alg(s, \hist_{t-1})_{a'}$ regardless of
  state $s$, (available) action $a$, or history $h_{t-1}$.
  $Q^*_M(s,a) > Q^*_M(s,a') + \alpha \Rightarrow \mathcal{L}(s,
  \hist_{t-1})_a \geq \mathcal{L}(s, \hist_{t-1})_{a'}$
  then follows immediately, which guarantees both fairness and
  \afair fairness.
\end{proof}

\begin{obs}
\label{obs:m-alphaa}
Let $M$ be an MDP and $M^\alpha$ the $\alpha$-restricted MDP of
$M$. Let $\pi$ be a policy in $M^\alpha$. Then $\pi$ is $\alpha$-action fair.
\end{obs}

\begin{proof}
  Assume $\pi$ is not $\alpha$-action fair. Then there must exist round $t$,
  state $s$, and action $a$ such that
  $Q^*_M(s,a) > Q^*_M(s,a') + \alpha$ and
  $\mathcal{L}(s, \hist_{t-1})_a < \mathcal{L}(s,
  \hist_{t-1})_{a'}$.
  Therefore $\mathcal{L}(s, \hist_{t-1})_{a'} > 0$, so $M^{\alpha}$
  must include action $a'$ from state $s$. But this is a
  contradiction, as in state $s$ $M^{\alpha}$ only includes actions
  $a'$ such that $Q^*_M(s,a') + \alpha \geq Q^*_M(s,a)$. $\pi$ is
  therefore $\alpha$-action fair.
\end{proof}

\begin{obs}
\label{obs:m-alpha-opt}
Let $M$ be an MDP and $M^\alpha$ the $\alpha$-restricted MDP of
$M$. Let $\pi^*$ be an optimal policy in $M^\alpha$. Then $\pi^*$ is
also optimal in $M$.
\end{obs}

\begin{proof}
  If $\pi^*$ is not optimal in $M$, then there exists a state $s$ and
  action $a$ such that
  $Q^*_M(s, a) > \E_{a^*(s)\sim\pi^*(s)} Q^*_M(s, a^*(s))$ where
  $a^*(s)$ is drawn from $\pi^*(s)$ and the expectation is taken over
  choices of $a^*(s)$. This is a contradiction because action $a$ is
  available from state $s$ in $M^\alpha$ by
  Definition~\ref{def:fair-mdp}.
\end{proof} 

%% file: fair-ecube-details.tex
\section{Omitted Details of \ecubep}
\label{sec:ecubep-deets}

We first formally define the exploitation MDP $\mknown$ and the
exploration MDP $\munknown$:
\begin{definition} [Definition 9,~\citet{KearnsS02}]
\label{def:exploit-induced-mdp}
Let $M = (\mathcal{S}_M, \mathcal{A}_M, P_M, R_M, T, \gamma)$ be an
MDP with state space $\mathcal{S}_M$ and let
$\known \subset \mathcal{S}_M$.  We define the \emph{exploration ~MDP}
$\mknown = (\mathcal{S}_{\mknown}, \mathcal{A}_M, P_{\mknown},
R_{\mknown}, T, \gamma)$ on $\known$ where
\begin{itemize}
\item $\mathcal{S}_{\mknown} = \known \cup \{s_0\}$.
\item For any state $s\in \known$,
  $\bar{R}_{\mknown}(s) = \bar{R}_{M}(s)$, rewards in $\mknown$ are
  deterministic, and $\bar{R}_{\mknown}(s_0) = 0$.
\item For any action $a$, $P_{\mknown}(s_0, a, s_0) = 1$. Hence, $s_0$
  is an absorbing state.
\item For any states $s_1, s_2 \in \known$ and any action $a$,
  $P_{\mknown}(s_1, a, s_2) = P_{M}(s_1, a, s_2)$, i.e.  transitions
  between states in $\known$ are preserved in $\mknown$.
\item For any state $s_1\in \known$ and any action $a$,
  $P_{\mknown}(s_1, a, s_0) = \Sigma_{s_2\notin \known}P_M(s_1, a,
  s_2)$.
  Therefore, all the transitions between a state in $\known$ and
  states not in $\known$ are directed to $s_0$ in $\mknown$.
\end{itemize}
\end{definition}

\begin{definition}[Implicit, \citet{KearnsS02}]
\label{def:induced-mdp-exp}
Given MDP $M$ and set of known states $\Gamma$, the
\emph{exploration~MDP} $\munknown$ on $\known$ is identical to the
exploitation~MDP $\mknown$ except for its reward
function. Specifically, rewards in $\munknown$ are deterministic as in
$\mknown$, but for any state $s\in \known$,
$\bar{R}_{\munknown}(s) = 0$, and $\bar{R}_{\munknown}(s_0) = \rmax$.
\end{definition}

We next define the approximation MDPs $\hmknown$ and
$\hmunknown$ which are defined over the same set of states and actions as 
in $\mknown$ and $\munknown$, respectively. 

Let $M$ be an MDP and $\known$ the set of known states of $M$. 
For any $s, s'\in \known$ and action $a\in A$, 
let $\hat P_{\mknown}(s,a,s')$ denote the empirical probability transition
estimates obtained from the visits to $s$. 
Moreover, for any state $s\in\known$ let $\bar{\hat {R}}_{\mknown}(s)$ denote the empirical
estimates of the average reward obtained from visits to s.
Then $\hmknown$ is identical to $\mknown$ except that: 
\begin{itemize}
\item in any known state $s\in\known$, 
$\hat{R}_{\hmknown}(s)=\bar{\hat {R}}_{\mknown}(s)$.
\item for any $s, s'\in \known$ and action $a\in A$, $P_{\hmknown}(s, a, s') = \hat{P}_{\mknown}(s, a, s')$.
\end{itemize}
Also $\hmunknown$ is identical to $\munknown$ except that: 
\begin{itemize}
\item for any $s, s'\in \known$ and action $a\in A$, $P_{\hmunknown}(s, a, s') = \hat{P}_{\munknown}(s, a, s')$.
\end{itemize}